\documentclass[letterpaper]{article} 
\usepackage{kr}  
\usepackage{times}  
\usepackage{helvet} 
\usepackage{courier}  
\usepackage[hyphens]{url}  
\usepackage{graphicx} 
\usepackage{graphicx}  
\usepackage{times}

\usepackage{acronym}
\usepackage{amsfonts}
\usepackage{amsmath}
\usepackage{amsthm}
\usepackage{listings}
\usepackage{mleftright}
\usepackage{paralist}
\usepackage{stmaryrd}

\usepackage[textsize=scriptsize,disable]{todonotes}

\acrodef{BAT}{basic action theory}
\acrodef{fd-BAT}{finite-domain basic action theory}
\acrodef{MTL}{Metric Temporal Logic}
\acrodef{STS}{symbolic transition system}
\acrodef{TA}{timed automaton}
\acrodef{MITL}{Metric Interval Temporal Logic}
\acrodef{LTL}{Linear Time Logic}

\newtheorem{definition}{Definition}
\newtheorem{lemma}{Lemma}
\newtheorem{theorem}{Theorem}
\newtheorem{proposition}{Proposition}

\usepackage[hidelinks]{hyperref}

\usetikzlibrary{arrows,positioning,shapes}

\definecolor{rwth-blue}{cmyk}{1,.5,0,0}\colorlet{rwth-lblue}{rwth-blue!50}\colorlet{rwth-llblue}{rwth-blue!25}
\definecolor{rwth-violet}{cmyk}{.6,.6,0,0}\colorlet{rwth-lviolet}{rwth-violet!50}\colorlet{rwth-llviolet}{rwth-violet!25}
\definecolor{rwth-purple}{cmyk}{.7,1,.35,.15}\colorlet{rwth-lpurple}{rwth-purple!50}\colorlet{rwth-llpurple}{rwth-purple!25}
\definecolor{rwth-carmine}{cmyk}{.25,1,.7,.2}\colorlet{rwth-lcarmine}{rwth-carmine!50}\colorlet{rwth-llcarmine}{rwth-carmine!25}
\definecolor{rwth-red}{cmyk}{.15,1,1,0}\colorlet{rwth-lred}{rwth-red!50}\colorlet{rwth-llred}{rwth-red!25}
\definecolor{rwth-magenta}{cmyk}{0,1,.25,0}\colorlet{rwth-lmagenta}{rwth-magenta!50}\colorlet{rwth-llmagenta}{rwth-magenta!25}
\definecolor{rwth-orange}{cmyk}{0,.4,1,0}\colorlet{rwth-lorange}{rwth-orange!50}\colorlet{rwth-llorange}{rwth-orange!25}
\definecolor{rwth-yellow}{cmyk}{0,0,1,0}\colorlet{rwth-lyellow}{rwth-yellow!50}\colorlet{rwth-llyellow}{rwth-yellow!25}
\definecolor{rwth-grass}{cmyk}{.35,0,1,0}\colorlet{rwth-lgrass}{rwth-grass!50}\colorlet{rwth-llgrass}{rwth-grass!25}
\definecolor{rwth-green}{cmyk}{.7,0,1,0}\colorlet{rwth-lgreen}{rwth-green!50}\colorlet{rwth-llgreen}{rwth-green!25}
\definecolor{rwth-cyan}{cmyk}{1,0,.4,0}\colorlet{rwth-lcyan}{rwth-cyan!50}\colorlet{rwth-llcyan}{rwth-cyan!25}
\definecolor{rwth-teal}{cmyk}{1,.3,.5,.3}\colorlet{rwth-lteal}{rwth-teal!50}\colorlet{rwth-llteal}{rwth-teal!25}
\definecolor{rwth-silver}{cmyk}{.39,.31,.32,.14}
\definecolor{rwth-gold}{cmyk}{.35,.46,.7,.35}

\tikzset{
  >=stealth',
  every node/.style={
    rectangle,
    rounded corners,
    draw=black,
    thick,
    text centered
  },
  edgelabel/.style={
    font=\scriptsize,
    draw=none
  },
  action/.style={
    thick,
    draw=black,
    text=black
  },
  sysevent/.style={
    thick,
    draw=blue,
    dotted,
    text=blue
  },
  exoevent/.style={
    thick,
    draw=black,
    text=black
  }
}
\tikzset{stn/.style={rectangle,draw,font=\small}}

\title{
	Controller Synthesis for Golog Programs over Finite Domains \\ with Metric Temporal Constraints
}
\author{
  Till Hofmann \and Gerhard Lakemeyer \\
  \affiliations
  Knowledge-Based Systems Group, RWTH Aachen University \\
  \emails
  \{hofmann, gerhard\}@kbsg.rwth-aachen.de
}

\hypersetup{
  pdftitle      = {Controller Synthesis for Golog Programs over Finite Domains with Metric Temporal Constraints},
  pdfauthor     = {Till Hofmann, Gerhard Lakemeyer},
  pdfkeywords   = {situation calculus, golog, temporal constraints, temporal logic, cognitive robotics}
}

\newcommand*{\textcite}[1]{\citeauthor{#1}~(\citeyear{#1})}
\newcommand*{\citet}[1]{\citeauthor{#1} \shortcite{#1}}

\newcommand*{\golog}{\textsc{Golog}}
\newcommand*{\congolog}{\textsc{ConGolog}}
\newcommand*{\indigolog}{\textsc{IndiGolog}}

\newcommand*{\es}{\texorpdfstring{\ensuremath{\mathcal{E \negthinspace S}}}{ES}}
\newcommand*{\esg}{\texorpdfstring{\ensuremath{\mathcal{E \negthinspace S \negthinspace G}}}{ESG}}
\newcommand*{\tesg}{\texorpdfstring{\ensuremath{\operatorname{\mathit{t-}}\negthinspace\mathcal{E \negthinspace S \negthinspace G}}}{t-ESG}}
\newcommand*{\mtl}{MTL}

\newcommand*{\eqdef}{\ensuremath\overset{def}{=}}

\newcommand*{\equivspace}{\ensuremath\,\equiv\;}

\newcommand*{\until}[1]{\ensuremath{\,\mathbf{U}_{#1}}\,}

\newcommand*{\tnext}[1]{\ensuremath{\mathbf{X}_{#1}}\if#1{\,}\fi}
\newcommand*{\tprev}[1]{\ensuremath{\mathbf{V}_{#1}}\if#1{\,}\fi}
\newcommand*{\fut}[1]{\ensuremath{\mathbf{F}_{#1}}\if#1{\,}\fi}
\newcommand*{\past}[1]{\ensuremath{\mathbf{P}_{#1}}\if#1{\,}\fi}
\newcommand*{\glob}[1]{\ensuremath{\mathbf{G}_{#1}}\if#1{\,}\fi}
\newcommand*{\hist}[1]{\ensuremath{\mathbf{H}_{#1}}\if#1{\,}\fi}
\newcommand*{\mi}[1]{\ensuremath{\mathit{#1}}}
\newcommand*{\la}{\langle}
\newcommand*{\ra}{\rangle}
\newcommand*{\final}{\ensuremath{\mathcal{F}^w}}
\mathchardef\mhyphen="2D 

\newcommand*{\timedom}{\mathbb{Q}}
\newcommand*{\dogolog}[2]{\mi{s\_#1}(#2); \mi{e\_#1}(#2)}

\newif\ifhideproofs

\ifhideproofs
\usepackage{environ}
\NewEnviron{killcontents}{}

\fi

\begin{document}

\maketitle

\begin{abstract}
Executing a Golog program on an actual robot typically requires
additional steps to account for hardware or software details of the
robot platform, which can be formulated as constraints on the program.
Such constraints are often temporal, refer to metric time, and require
modifications to the abstract Golog program.  We describe how to
formulate such constraints based on a modal variant of the Situation
Calculus. These constraints connect the abstract program with the
platform models, which we describe using timed automata.  We show that
for programs over finite domains and with fully known initial state, the
problem of synthesizing a controller that satisfies the constraints
while preserving the effects of the original program can be reduced to
MTL synthesis.  We do this by constructing a timed automaton from the
abstract program and synthesizing an MTL controller from this automaton,
the platform models, and the constraints. We prove that the synthesized
controller results in execution traces which are the same as those of
the original program, possibly interleaved with platform-dependent
actions, that they satisfy all constraints, and that they have the same
effects as the traces of the original program. By doing so, we obtain a
decidable procedure to synthesize a controller that satisfies the
specification while preserving the original program.
\end{abstract}

\section{Introduction}
While \golog{} \cite{levesqueGOLOGLogicProgramming1997}, an agent programming language based on the Situation Calculus
\cite{mccarthySituationsActionsCausal1963,reiterKnowledgeActionLogical2001}, allows a clear and abstract specification
of an agent's behavior, executing a \golog{} program on a real robot often creates additional issues.  Typically, the
robot's platform requires additional constraints that are ignored when designing a \golog{} program.
As an example, a robot may need to calibrate its arm before it can use it.
One way to deal with such platform constraints is to split the reasoning into two
parts~\cite{hofmannConstraintbasedOnlineTransformation2018}: First, an abstract \golog{} program specifies the intended
behavior of the robot, without taking the robot platform into account. In a second step, the platform is considered by
transforming the abstract program into a program that is executable on the particular platform, given a model
of the platform and temporal constraints that connect the platform with the plan.

In this paper, we propose a method for such a transformation: We model the robot platform with a \acf{TA} and formulate
constraints with \tesg{}~\cite{hofmannLogicSpecifyingMetric2018}, a modal variant of the Situation Calculus extended
with temporal operators and metric time.  We then synthesize a controller that executes the abstract program, but also
inserts additional platform actions to satisfy the platform constraints. To do so, we restrict the \golog{} program to a
finite domain, finite traces, and a fully known initial state. This allows us to reduce the controller synthesis problem
to the \mtl{} control problem, which has been shown to be decidable \cite{bouyerControllerSynthesisMTL2006}.
Furthermore, for the purpose of this paper, we only use time to formulate temporal constraints on the robot platform and
we restrict programs to untimed programs, i.e., in contrast to programs in \cite{hofmannLogicSpecifyingMetric2018}, a
program may not refer to time and action preconditions and effects are time-independent. We will revisit these restrictions
in the concluding section.

In the following, we first give an overview on  the Situation Calculus and \golog{} and related work in
\autoref{sec:foundations} and summarize \tesg{} in \autoref{sec:timed-esg}.
In \autoref{sec:mtl-synthesis}, we describe timed automata and \acf{MTL}, before we summarize the
\ac{MTL} synthesis problem. We explain how to transform a \golog{} program over a finite domain
with a complete initial state into a \ac{TA} in \autoref{sec:pta} and how to model a robot platform with a \ac{TA} and
temporal constraints in \autoref{sec:platform-models}. Both \ac{TA} and the constraints are then used in
\autoref{sec:synthesis} to synthesize a controller that executes the program while satisfying all constraints.
We conclude in \autoref{sec:conclusion}.

\section{Related Work}\label{sec:foundations}
The Situation Calculus \cite{mccarthySituationsActionsCausal1963,reiterKnowledgeActionLogical2001} is a first-order
logic for representing and reasoning about actions.  Following \citeauthor{reiterKnowledgeActionLogical2001}, action
preconditions and effects as well as information about the initial situation are then encoded as so-called \emph{Basic
Action Theories (BATs)}.  The action programming language \golog{} \cite{levesqueGOLOGLogicProgramming1997} and its
concurrent variant \congolog{} \cite{degiacomoConGologConcurrentProgramming2000} are based on
the Situation Calculus and offer imperative programming constructs such as sequences of actions and iteration as well
as non-deterministic branching and non-deterministic choice.  The semantics of \golog{} and its on-line variant
\indigolog{} can be specified in terms of transitions \cite{degiacomoIndiGologHighlevelProgramming2009}.
The logic \es{} \cite{lakemeyerSemanticCharacterizationUseful2011} is a modal variant of the Situation Calculus which gets rid of
explicit situation terms and uses modal operators instead.
The logic \esg{} \cite{classenLogicNonTerminatingGolog2008,classenPlanningVerificationAgent2013} is a temporal extension
of \es{} and used for the verification of \golog{} programs.  It specifies program transition semantics similar to the
transition semantics of \indigolog{} and extends \es{} with the temporal operators \tnext{} (\emph{next}) and \until{}
(\emph{until}). The logic \tesg{}~\cite{hofmannLogicSpecifyingMetric2018} extends \esg{} with metric time and timing
constraints on the \emph{until} operator.


\ac{MTL} \cite{koymansSpecifyingRealtimeProperties1990} is an extension of \ac{LTL} with metric time, which allows
expressions such as $\fut{\leq c}$, meaning \emph{eventually within time $c$}.  In \ac{MTL}, formulas are interpreted
over \emph{timed words} or \emph{timed state sequences}, where each state specifies which propositions are true, and
each state has an associated time value.  Depending on the choice of the state and time theory, the satisfiability
problem for \ac{MTL} becomes undecidable \cite{alurRealTimeLogicsComplexity1993}.  However, both for finite words and
for a pointwise semantics, it has been shown to be decidable
\cite{ouaknineDecidabilityMetricTemporal2005,ouaknineRecentResultsMetric2008}.

Similar to the proposed approach, \textcite{schifferSelfMaintenanceAutonomousRobots2010} extend \golog{} for
self-maintenance by allowing temporal constraints using Allen's Interval Algebra
\cite{allenMaintainingKnowledgeTemporal1983}.  Those constraints are resolved on-line by interleaving the original
program with maintenance actions.  Closely related is also the work by \textcite{finziRepresentingFlexibleTemporal2005},
who propose a hybrid approach of temporal constraint reasoning and reasoning about actions based on the Situation
Calculus.  They also allow constraints based on Allen's Interval Algebra, which are translated into a temporal
constraint network. \citeauthor{degiacomoSynthesisLTLLDL2015} describe a synthesis method for \ac{LTL} and LDL
specifications over finite traces~\cite{degiacomoSynthesisLTLLDL2015}. Similar to \ac{MTL} synthesis, they partition the
propositions in controllable and uncontrollable symbols and use games to synthesize a controller.
Based on $\mi{LTL}_f$ synthesis, \citeauthor{heReactiveSynthesisFinite2017} describe a synthesis method that controls a
robot against uncontrollable environment actions under resource constraints~\cite{heReactiveSynthesisFinite2017}. They
model the underlying planning problem as a graph, where each vertex describes the state of the world and each edge
corresponds to an action, either by the agent or by the environment. In contrast to this work, they do not allow metric
temporal constraints.

\section{Timed \esg{}}\label{sec:timed-esg}

In this section, we summarize the syntax and semantics of \tesg{} \cite{hofmannLogicSpecifyingMetric2018}, which is
based on \esg{}~\cite{classenLogicNonTerminatingGolog2008} and \es{}~\cite{lakemeyerSemanticCharacterizationUseful2011}, modal
variants of the Situation Calculus. We refer to \cite{hofmannLogicSpecifyingMetric2018} for a more complete description.

The language has two sorts: {\em object} and {action}. A special feature inherited from \es{} is the use of countably
infinite sets of {\em standard names} for both sorts. Standard object names syntactically look like constants, but are
intended to be isomorphic with the set of all objects of the domain. In other words, standard object names can be
thought of as constants that satisfy the unique name assumption and domain closure for objects.
We assume that object standard names include the rational numbers (including $\infty$) as a subsort.
Action standard names
are function symbols of any arity whose arguments are standard object names. Examples are $\mi{pick}(o)$ and
$\mi{goto}(l_1,l_2)$ for picking up an object and going from one location to another, respectively. Again, standard action names range
over all actions and satisfy the unique name assumption and domain closure for actions. One advantage of using standard
names is that quantifiers can be understood substitutionally when defining the semantics. For simplicity, we do not
consider function symbols other than actions. Formally the language is defined as follows:

\subsection{Syntax}
\begin{definition}[Symbols of \tesg{}]
The symbols of the language are from the following vocabulary:
\begin{enumerate}
  \item object variables $x_1, x_2, x_3, \ldots, y_1, \ldots$,
  \item action variables $a, a_1, a_2, a_3, \ldots$,
  \item object standard names $\mathcal{N}_O = \{ o_1, o_2, o_3, \ldots \}$,
  \item action standard names $\mathcal{N}_A = \{ p_1, p_2, p_3, \ldots \}$,
  \item fluent predicates of arity $k$: $\mathcal{F}^k: \{ F_1^k, F_2^k, \ldots \}$, e.g., $\mi{Holding(o)}$; we assume this list contains the distinguished predicate $\mi{Poss}$
  \item rigid predicates of arity $k$: $\mathcal{G}^k = \{ G_1^k, G_2^k, \ldots \}$,
  \item open, closed, and half-closed intervals, e.g., $[1, 2]$, with rational numbers as interval endpoints,
  \item connectives and other symbols: $=$, $\wedge$, $\vee$, $\neg$, $\forall$, $\square$, $[\cdot]$, $\llbracket \cdot \rrbracket$,
    $\until{I}$ (with interval $I$).
\end{enumerate}
\end{definition}

We denote the set of standard names as $\mathcal{N} = \mathcal{N}_O \cup \mathcal{N}_A$.


\begin{definition}[Terms of \tesg{}]
  The set of terms of \tesg{} is the least set such that
  \begin{inparaenum}[(1)]
    \item every variable is a term of the corresponding sort,
    \item every standard name is a term.
  \end{inparaenum}
\end{definition}

\begin{definition}[Formulas]
  The \emph{formulas of \tesg{}}, consisting of \emph{situation formulas} and \emph{trace formulas}, are the least set such that
  \begin{enumerate}
    \item if $t_1,\ldots,t_k$ are terms and $P$ is a $k$-ary predicate symbol,
      then $P(t_1,\ldots,t_k)$ is a situation formula,
    \item if $t_1$ and $t_2$ are terms,
      then $(t_1 = t_2)$ is a situation formula,
    \item if $\alpha$ and $\beta$ are situation formulas,
      $x$ is a variable, $\delta$ is a program (defined below), and
      $\phi$ is a trace formula,
      then $\alpha \wedge \beta$, $\neg \alpha$, $\forall x.\, \alpha$,
      $\square\alpha$, $[\delta]\alpha$,
      and $\llbracket \delta \rrbracket \phi$ are situation formulas,
    \item if $\alpha$ is a situation formula, it is also a trace formula,
    \item if $\phi$ and $\psi$ are trace formulas, $x$ is a variable, and $I$ is
      an interval,
      then $\phi \wedge \psi$, $\neg \phi$, $\forall x.\,\phi$,
      and $\phi \until{I} \psi$ are also trace formulas.
  \end{enumerate}
\end{definition}

A predicate symbol with standard names as arguments is called a \emph{primitive formula}, and we denote the set of
primitive formulas as $\mathcal{P}_F$.  We read $\square \alpha$ as ``$\alpha$ holds after executing any sequence of
actions'', $[\delta] \alpha$ as ``$\alpha$ holds after the execution of program $\delta$'',
$\llbracket \delta \rrbracket \alpha$ as ``$\alpha$ holds during the execution of program $\delta$'',
$\phi \until{I} \psi$ as ``$\phi$ holds \emph{until} $\psi$ holds, and $\psi$ holds within interval $I$''.

  A formula is called {\em static}  if it contains no $[\cdot]$,
      $\square$, or $\llbracket \cdot \rrbracket$ operators. It is called {\em fluent} if it is static and does not mention $\mi{Poss}$.

We also write $< c$, $\leq c$,  $= c$, $> c$, and $\geq c$ for the respective intervals
$[0, c)$, $[0, c]$, $[c, c]$, $(c,\infty)$, and $[c, \infty)$.
We use the short-hand notations
$\fut{I} \phi \eqdef (\top \until{I} \phi)$ (\emph{future}) and
$\glob{I}\phi \eqdef \neg \fut{I} \neg \phi$
(\emph{globally}).
For intervals, $c + [s,e]$ denotes the interval
$[s+c, e+c]$, similarly for $c + (s, e)$, $c + [s, e)$, and $c + (s, e]$.
We also omit the interval $I$ if $I = [0, \infty)$, e.g., $\phi \until{} \psi$ is short for $\phi \until{[0,\infty)} \psi$.

\newcommand*{\smid}{\ensuremath{\:\mid\:}}

Finally we define the syntax of \golog{} programs referred to by the operators $[\delta]$ and $\llbracket \delta \rrbracket$:

\begin{definition}[Programs]
  \[
    \delta ::= t \smid \alpha? \smid \delta_1 ; \delta_2 \smid \delta_1|\delta_2
    \smid \pi x.\, \delta
    \smid \delta_1 \| \delta_2 \smid \delta^*
  \]
  where $t$ is an action term and $\alpha$ is a static situation formula. A
  program consists of actions $t$, tests $\alpha?$, sequences
  $\delta_1;\delta_2$, nondeterministic branching $\delta_1 | \delta_2$,
  nondeterministic choice of argument $\pi x.\, \delta$, interleaved concurrency
  $\delta_1 \| \delta_2$, and nondeterministic iteration $\delta^*$.
\end{definition}

We also use the abbreviation $\mi{nil} \eqdef \top?$ for the empty program that always succeeds.
We remark that the above program constructs are a proper subset of the original \congolog~
\cite{degiacomoConGologConcurrentProgramming2000}. We have left out other constructs such as prioritized concurrency for
simplicity.

\subsection{Semantics}

\begin{definition}[Timed Traces]
  A \emph{timed trace} is a finite timed sequence of action standard names with monotonically non-decreasing time.
Formally, a trace $\pi$ is a mapping $\pi:\mathbb{N} \rightarrow \mathcal{P}_A \times \timedom$, and for any $i,j \in \mathbb{N}$ with $\pi(i) = \left(\sigma_i,t_i\right)$, $\pi(j) = \left(\sigma_j,t_j\right)$ : If $i < j$, then $t_i \leq t_j$.
\end{definition}

We denote the set of timed traces as $\mathcal{Z}$.
For a timed trace $z = \left(a_1,t_1\right) \ldots \left(a_k,t_k\right)$, we define $\mi{time}(z) \eqdef
t_k$ for $k > 0$ and $\mi{time}(\la\ra) \eqdef 0$,
i.e., $\mi{time}(z)$ is the time value of the last action in $z$.
We define the timed trace $z^0$ where all actions occur at time $0$ as
$z^0 = \left(a_1, 0\right)\left(a_2, 0\right) \ldots \left(a_n, 0\right)$.

\begin{definition}[World]
  Intuitively, a world $w$ determines the truth of fluent predicates, not just initially, but after any (timed) sequence of actions.
  Formally, a world $w$ is a mapping $\mathcal{P}_F \times \mathcal{Z} \rightarrow \{ 0, 1 \}$. If $G$ is a rigid
  predicate symbol, then for all $z$ and $z'$ in $\mathcal{Z}$, $w[G(n_1,\ldots,n_k),z]=w[G(n_1,\ldots,n_k),z']$.  
\end{definition}

Similar to \es{} and \esg{}, the truth of a fluent after any sequence of actions is determined by a world $w$.
Different from $\es{}$ and $\esg{}$, we require all traces referred to by a world to contain time values for each
action.  This also means that in the same world, a fluent predicate $F(\vec{n})$ may have a different value after the
same sequence of actions if the actions were executed at different times, i.e.,
$w[F(\vec{n}, \langle \left(a_1,1\right) \rangle]$ may have a different value than
$w[F(\vec{n}, \langle \left(a_1,2\right) \rangle]$. However, for simplicity the actions considered in basic action
theories (see Section~\ref{sec:BAT}) do not make use of this feature.


Next we define the transitions programs may take in a given world $w$. In two places these refer to the satisfaction of
situation formulas (see Definition~\ref{def:tesg-truth} below).
\newcommand{\warrow}{\ensuremath{\overset{w}{\rightarrow}}}
\begin{definition}[Program Transition Semantics]\label{def:trans}
  The transition relation \warrow{} among configurations, given
  a world $w$, is the least set satisfying
  \begin{enumerate}
    \item $\la z, a \ra \warrow \la z \cdot \left(p,t\right), \mi{nil} \ra$, if $t \geq \mi{time}(z)$, and $w, z  \models \mi{Poss}(p)$
    \item $\la z, \delta_1;\delta_2 \ra \warrow
      \la z \cdot p, \gamma;\delta_2 \ra$,
      if $\la z,\delta_1 \ra \warrow \la z \cdot p, \gamma \ra$,
    \item $\la z, \delta_1;\delta_2 \ra \warrow \la z \cdot p, \delta' \ra$
      if $\la z, \delta_1 \ra \in \mathcal{F}^w$ and
      $\la z, \delta_2 \ra \warrow \la z \cdot p, \delta' \ra$
    \item $\la z, \delta_1 | \delta_2 \ra \warrow \la z \cdot p, \delta' \ra$
      if $\la z, \delta_1 \ra \warrow \la z \cdot p, \delta' \ra$
      or $\la z, \delta_2 \ra \warrow \la z \cdot p, \delta' \ra$
    \item $\la z, \pi x.\, \delta \ra \warrow \la z \cdot p, \delta' \ra$,
      if $\la z, \delta^x_n \ra \warrow \la z \cdot p, \delta' \ra$ for some
      $n\in\mathcal{N}_x$
    \item $\la z, \delta^* \ra \warrow \la z \cdot p, \gamma; \delta^* \ra$ if
      $\la z, \delta \ra \warrow \la z \cdot p, \gamma \ra$
    \item $\la z, \delta_1 \| \delta_2 \ra \warrow \la z \cdot p, \delta' \|
      \delta_2 \ra$ if $z, \delta_1 \warrow \la z \cdot p, \delta' \ra$
    \item $\la z, \delta_1 \| \delta_2 \ra \warrow \la z \cdot p, \delta_1 \|
      \delta' \ra$ if $z, \delta_2 \warrow \la z \cdot p, \delta' \ra$
  \end{enumerate}
  The set of final configurations \final{} is the smallest set such that
  \begin{enumerate}
    \item $\la z, \alpha? \ra \in \final$ if $w, z \models \alpha$,
    \item $\la z, \delta_1;\delta_2 \ra \in \final$
      if $\la z, \delta_1 \ra \in \final$ and $\la z, \delta_2 \ra \in \final$
    \item $\la z, \delta_1 | \delta_2 \ra \in \final$
      if $\la z, \delta_1 \ra \in \final$,
      or $\la z, \delta_2 \ra \in \final$
    \item $\la z, \pi x.\,\delta \ra \in \final$
      if $\la z, \delta^x_n \ra \in \final$ for some $n \in \mathcal{N}_x$
    \item $\la z, \delta^* \ra \in \final$
    \item $\la z, \delta_1 \| \delta_2 \ra \in \final$
      if $\la z, \delta_1 \ra \in \final$
      and $\la z, \delta_2 \ra \in \final$
  \end{enumerate}
\end{definition}

The program transition semantics is very similar to the semantics of \esg{}.
The only difference is in Rule 1, which has an additional constraint on the time, and which requires the action to be executable.

\begin{definition}[Program Traces]\label{def:program-trace}
  Given a world $w$ and a finite sequence of action standard names $z$, the set
  $\|\delta\|^z_w$ of \emph{finite timed traces of a program $\delta$} is
  \begin{multline*}
    \|\delta\|^z_w =
    \\
    \{ z' \in \mathcal{Z} \mid
      \la z, \delta \ra \warrow^* \la z \cdot z', \delta' \ra
    \text{ and } \la z \cdot z', \delta' \ra \in \final \}
  \end{multline*}
%
\end{definition}

\begin{definition}[Truth of Situation and Trace Formulas]\label{def:tesg-truth}
  Given a world $w \in \mathcal{W}$ and a situation formula $\alpha$, we define
  $w \models \alpha$ as $w,\langle\rangle \models \alpha$, where for any $z \in
  \mathcal{Z}$:
  \begin{enumerate}
    \item $w, z \models F(n_1,\ldots,n_k)$ iff $w[F(n_1,\ldots,n_k),z] = 1$;
    \item $w, z \models (n_1 = n_2)$ iff $n_1$ and $n_2$ are identical;
    \item $w, z \models \alpha \wedge \beta$ iff $w, z \models \alpha$ and
      $w, z \models \beta$;
    \item $w, z \models \neg \alpha$ iff $w, z \not\models \alpha$;
    \item $w, z \models \forall x.\, \alpha$ iff $w, z \models \alpha^x_n$ for
      every standard name of the right sort;
    \item $w, z \models \square \alpha$ iff $w, z \cdot z' \models \alpha$
      for all $z' \in \mathcal{Z}$;
    \item $w, z \models [\delta]\alpha$ iff for all finite
      $z' \in \|\delta\|^z_w$, $w, z \cdot z' \models \alpha$;
    \item $w, z \models \llbracket \delta \rrbracket \phi$ iff for all $\tau \in
      \|\delta\|^z_w$, $w, z, \tau \models \phi$.
  \end{enumerate}
  Intuitively, $[\delta]\alpha$ means that \emph{after every execution} of $\delta$, the situation formula $\alpha$ is true.
  $\llbracket \delta \rrbracket \phi$ means that \emph{during every execution} of $\delta$, the trace formula $\phi$ is true.

  The truth of trace formulas $\phi$ is defined as follows for $w \in
  \mathcal{W}$, $z, \tau \in \mathcal{Z}$:
  \begin{enumerate}
    \item $w, z, \tau \models \alpha$ iff $w, z \models \alpha$ and $\alpha$ is
      a situation formula;
    \item $w, z, \tau \models \phi \wedge \psi$ iff $w, z, \tau \models \phi$
      and $w, z, \tau \models \psi$;
    \item $w, z, \tau \models \neg \phi$ iff $w, z, \tau \not\models \phi$;
    \item $w, z, \tau \models \forall x.\, \phi$ iff
      $w, z, \tau \models \phi^x_n$ for all $n \in \mathcal{N}_x$;
    \item $w, z, \tau \models \phi \until{I} \psi$ iff there is a $z_1 \neq \la\ra$ such that
      \begin{enumerate}
        \item $\tau = z_1 \cdot \tau'$,
        \item $\mi{time}(z_1) \in \mi{time}(z) + I$,
        \item $w, z \cdot z_1, \tau' \models \psi$,
        \item for all $z_2 \neq z_1$ with $z_1 = z_2 \cdot z_3$: $w, z \cdot z_2,
          z_3 \cdot \tau' \models \phi$.
      \end{enumerate}
  \end{enumerate}
\end{definition}



\begin{definition}[Validity]
  A situation formula $\alpha$ is \emph{valid} (written $\models \alpha$) iff for every world $w,$~$w \models
  \alpha$.
  A trace formula $\phi$ is \emph{valid} ($\models \phi$) iff for every world $w$ and every trace $\tau$,~ $w, \langle\rangle, \tau \models \phi$.
\end{definition}

\subsection{Basic Action Theories}
\label{sec:BAT}

A \acf{BAT} defines the preconditions and effects of all actions of the domain, as well as the initial state:
\begin{definition}[\acl*{BAT}]
  Given a finite set of fluent predicates $\mathcal{F}$, a set $\Sigma \subseteq \tesg$
  of sentences is called a \acf{BAT} over $\mathcal{F}$ iff
  $\Sigma = \Sigma_0 \cup \Sigma_\text{pre} \cup \Sigma_\text{post}$, where
  $\Sigma$ mentions only fluents in $\mathcal{F}$ and
  \begin{enumerate}
    \item $\Sigma_0$ is any set of fluent sentences,
    \item $\Sigma_\text{pre}$ consists of a single sentence
      of the form $\square
      \mathit{Poss}(a) \equivspace \pi$, where $\pi$ is a fluent formula with free variable $a$.\footnote{Free variables
        are implicitly universal quantified from the outside. The modality $\Box$ has lower syntactic precedence than the connectives, and $[\cdot]$ has the highest priority.}
    \item $\Sigma_\text{post}$ is a set of sentences, one for each fluent predicate $F \in \mathcal{F}$, of the form $\square[a]F(\vec{x}) \equivspace
      \gamma_F$.
  \end{enumerate}
\end{definition}

The set $\Sigma_0$ describes the initial state, $\Sigma_\text{pre}$ defines the preconditions of all actions of the
domain, and $\Sigma_\text{post}$ defines action effects by specifying for each fluent of the domain whether the fluent
is true after doing some action $a$.

We will also consider \ac{BAT}s restricted to a finite domain of actions and objects:
\begin{definition}[Finite-domain \ac{BAT}]
  We call a \ac{BAT} $\Sigma$ a \emph{\ac{fd-BAT}} iff
  \begin{enumerate}
    \item each $\forall$ quantifier in $\Sigma$ occurs as $\forall x.\, \tau_i(x) \supset \phi(x)$,
      where $\tau_i$ is a rigid predicate, $i=o$ if $x$ is of sort object, and $i=a$ if $x$ is of sort action; \label{fd:quantifiers}
    \item $\Sigma_0$ contains axioms
      \begin{itemize}
      \item $\tau_o(x) \equiv (x = n_1 \vee x = n_2 \vee \ldots \vee x = n_k)$ and
      \item $\tau_a(a) \equiv (a = m_1 \vee a = m_2 \vee \ldots \vee a = m_l)$
      \end{itemize}
       where the $n_i$ and $m_j$ are object and action standard names,
       respectively. Also each $m_j$ may only mention object standard names $n_i$.
  \end{enumerate}
\end{definition}
We call a formula $\alpha$ that only mentions symbols and standard names from $\Sigma$ \emph{restricted to $\Sigma$}
and we denote the set of primitive formulas restricted to $\Sigma$ as $\mathcal{P}_\Sigma$ and the action standard names
mentioned in $\Sigma$ as $A_\Sigma$.
We also write $\exists x\mathbf{:}i.\,\phi$ for $\exists x.\, \tau_i(x) \wedge \phi$ and $\forall x\mathbf{:}i.\,\phi$
for $\forall x.\, \tau_i(x) \supset \phi$.
Since an \ac{fd-BAT} essentially restricts the domain to be finite, quantifiers of type object can be understood as abbreviations:
\begin{align*}
  \exists x\mathbf{:}\tau_o. \phi &\eqdef \bigvee_{i=1}^k \phi^x_{n_i},
  \\
  \forall x\mathbf{:}\tau_o. \phi &\eqdef \bigwedge_{i=1}^k \phi^x_{n_i},
\end{align*}
and similarly for quantifiers of type action.

In addition to a finite domain, we also restrict a \ac{BAT} such that it completely determines the initial situation:
\begin{definition}[determinate \ac{BAT}]
  A \ac{fd-BAT} $\Sigma$ is \emph{determinate} iff every for atomic formula $\alpha$ restricted to $\Sigma$, either
  $\Sigma_0 \models \alpha$ or $\Sigma_0 \models \neg\alpha$.
\end{definition}

Next, given a world $w$, we define a world $w_\Sigma$ that is consistent with $\Sigma$:
\begin{definition}
For any world $w$ and basic action theory $\Sigma$, we define a world $w_\Sigma$
which is like $w$ except that it satisfies the $\Sigma_\text{pre}$ and
$\Sigma_\text{post}$ sentences of $\Sigma$.
\end{definition}

\begin{lemma}[\cite{lakemeyerSemanticCharacterizationUseful2011}]
  For any $w$, $w_\Sigma$ exists and is uniquely defined.
\end{lemma}

For a determinate \ac{BAT} over a set of fluent predicates $\mathcal{F}$, we can show that $\Sigma$ fully determines the
truth of every fluent $f \in \mathcal{F}$, not only initially, but after any sequence of actions:
\begin{lemma}\label{lma:det-bat}
  Let $\Sigma$ be a determinate \ac{BAT} over $\mathcal{F}$, $\delta$ a program over $\Sigma$ and $w, w'$ two worlds,
  and $z \in \mathcal{Z}$ a finite trace such that
  $\la \la\ra, \delta \ra \overset{w_\Sigma}{\longrightarrow^*} \la z, \delta' \ra$.
  Then
  \begin{enumerate}
    \item $\la \la\ra, \delta \ra \overset{w'_\Sigma}{\longrightarrow^*} \la z, \delta' \ra$, \label{lma:det-bat:traces}
    \item for every primitive formula $F\left(\vec{t}\right)$ with $F \in \mathcal{F}$:
      $w_\Sigma[F(\vec{t}),z] = w'_\Sigma[F(\vec{t}), z]$
      \label{lma:det-bat:fluents}
  \end{enumerate}
\end{lemma}

\begin{proof}
  By induction over the length of $z$.
  \begin{itemize}
    \item Let $z = \la\ra$. By definition of a determinate \ac{BAT}, we know that
      $w_\Sigma[F(\vec{t}), \la\ra] = 1 \Leftrightarrow w'_\Sigma[F(\vec{t}), \la\ra] = 1$.
    \item Let $z = z' \cdot \left(p, t\right)$. By induction, for each atomic formula $\alpha$, $w_\Sigma [\alpha,
      z'] = w'_\Sigma [\alpha, z']$, and thus, for each fluent situation formula $\gamma$, $w_\Sigma, z' \models \gamma$
      iff $w'_\Sigma, z' \models \gamma$.
      Furthermore, we know from $\la \la\ra, \delta \ra \overset{w_\Sigma}{\longrightarrow^*} \la z, \delta' \ra$ that
      for some $z', \delta''$,
      $\la z', \delta'' \ra \overset{w_\Sigma}{\longrightarrow} \la z, \delta' \ra$ and thus
      $w_\Sigma, z'  \models  \mi{Poss}(p)$.
      As both $w_\Sigma$ and $w'_\Sigma$ satisfy $\Sigma_\text{pre}$, it follows that
      $w'_\Sigma, z' \models  \mi{Poss}(p)$ and therefore
      $\la \la\ra, \delta \ra \overset{w'_\Sigma}{\longrightarrow^*} \la z, \delta' \ra$.
      As both $w_\Sigma$ and $w'_\Sigma$ satisfy $\Sigma_\text{post}$ and there is a successor state axiom for
      each $F$, it follows that $w_\Sigma[F(\vec{t}), z] = 1$ iff $w_\Sigma, z' \models \gamma_F(\vec{t})$ and
      $w'_\Sigma[F(\vec{t}), z] = 1$ iff $w'_\Sigma, z' \models \gamma_F(\vec{t})$ and thus
      $w_\Sigma[F(\vec{t}), z] = 1 \Leftrightarrow w'_\Sigma[F(\vec{t}), z] = 1$.
      \qedhere
  \end{itemize}
\end{proof}

%
%

In fact, we can show that $\Sigma$ fully determines possible traces of $\delta$, as well as the truth of any
formula restricted to $\Sigma$:
\begin{theorem}\label{thm:unique-sigma-worlds}
  Let $\Sigma$ be a determinate \ac{BAT}, $\delta$ a program over $\Sigma$ and $w, w'$ two worlds, and
  $z \in \|\delta\|_{w_\Sigma}$,
  $\alpha$ a situation formula and $\phi$ a trace formula, both restricted to $\Sigma$.
  Then:
  \begin{enumerate}
    \item $z \in \|\delta\|_{w'_\Sigma}$
    \item $w_\Sigma \models [\delta]\alpha \Leftrightarrow w'_\Sigma \models [\delta]\alpha$
    \item $w_\Sigma \models \llbracket\delta\rrbracket\phi \Leftrightarrow w'_\Sigma \models
      \llbracket\delta\rrbracket\phi$
  \end{enumerate}
\end{theorem}

\begin{proof}
  Follows from \autoref{lma:det-bat}.
\end{proof}

For the purpose of this paper and in contrast to \cite{hofmannConstraintbasedOnlineTransformation2018}, we do not have
distinguished function symbols $\mi{now}$ and $\mi{time}$ that allow referring to time in a situation formula. In
particular, this means that we cannot define time-dependent preconditions or effects in a \ac{BAT}. Thus, time is only
relevant for the truth of trace formulas. Also, a program's traces are not restricted with respect to time:
\begin{proposition}\label{lma:0-equivalence}
  Given a \ac{BAT} $\Sigma$, a program $\delta$,
  and a world $w$.
  Let $\tau_1, \tau_2$ be two traces with
  $\tau_1(i) = \left(a_i, t_i\right)$, $\tau_2(i) = \left(a_i, t'_i\right)$ for every $i$ (i.e., they contain the same
  action symbols but different time points).
  Then $\tau_1 \in \|\delta\|_{w_\Sigma}$ iff $\tau_2 \in \|\delta\|_{w_\Sigma}$.
\end{proposition}

\subsubsection{A Simple Carrier Bot}
With the following determinate \ac{fd-BAT}, we describe a simple carrier bot that is able to move to locations and pick up
objects:
\begin{align}
  &\square \mi{Poss}(a) \equiv \nonumber
  \\
  &\exists s\mathbf{:}o \exists g\mathbf{:}o .\, a = \mi{s\_goto}(s,g)  \label{eqn:poss-start-goto}
  \wedge \neg \exists a'\mathbf{:}a .\, \mi{Perf}(a') \\
  &\vee \exists s\mathbf{:}o \exists g\mathbf{:}o .\, a = \mi{e\_goto}(s,g) \label{eqn:poss-end-goto} \wedge
  \mi{Perf}(\mi{goto}(s,g)) \\
  &\vee \exists o\mathbf{:}o,l\mathbf{:}o .\, a = \mi{s\_pick}(o) \label{eqn:poss-start-pick}
  \wedge \neg \exists a'\mathbf{:} a .\, \mi{Perf}(a') \\
  &\quad \wedge \mi{RAt}(l) \wedge \mi{At}(o,l) \nonumber \\
  &\vee \exists o\mathbf{:}o .\, a = \mi{e\_pick}(o) \wedge \label{eqn:poss-end-pick}
  \mi{Perf}(pick(o))
\end{align}
The precondition axioms state that it is possible to start the \texttt{goto} action ($\mi{s\_goto}$) if the robot is not
performing any action (\autoref{eqn:poss-start-goto}), it can stop the \texttt{goto} action if it is currently
performing it (\autoref{eqn:poss-end-goto}). Furthermore, it can start picking up an object if it is not performing any
other action and it is at the same position as the object (\autoref{eqn:poss-start-pick}). Finally, it can stop
picking if it is currently performing a \texttt{pick} action (\autoref{eqn:poss-end-pick}).

By splitting actions into start and stop actions, we can execute multiple actions concurrently. We will later insert
platform actions that are executed in addition and concurrent to the program's actions. Also, splitting actions into
start and stop actions allows us to model that only the start but not the end of an action is under the robot's
control. In \autoref{sec:synthesis}, we will let the environment control all end actions, i.e., the environment will
decide when an action ends.

In addition to the precondition axioms, we also define successor state axioms for all fluents of the domain:
\begin{align}
  \square[a]&\mi{RAt}(l) \equivspace \label{eqn:ssa-robotat}
  \exists s\mathbf{:}o.\, a = \mi{e\_goto}(s,l)) \\
            &\quad \vee \mi{RAt}(l) \wedge \neg \exists s'\mathbf{:}o\, \exists g'\mathbf{:}o .\, a = \mi{s\_goto}(s',g') \nonumber \\
  \square[a]&\mi{At}(p,l) \equivspace \label{eqn:ssa-at}
            \mi{At}(p,l)
            \wedge a \neq \mi{s\_pick}(p)
            \\
  \square[a]&\mi{Holding}(p) \equivspace \label{eqn:ssa-holding}
            a = \mi{e\_pick}(p) \vee \mi{Holding}(p)
            \\
  \square[a]&\mi{Perf}(a') \equivspace \label{eqn:ssa-performing}
            \\ &\nonumber
            \exists s\mathbf{:}o \exists g \mathbf{:}o.\, \left[a = \mi{s\_goto}(s,g) \right] \nonumber
            \vee \exists o \left[a = \mi{s\_pick}(o) \right]
            \nonumber
            \\
            &\vee
            \mi{Perf}(a') \wedge \nonumber
            \neg\exists s\mathbf{:}o \exists g \mathbf{:}o \left[ a = \mi{e\_goto}(s,g) \right] \nonumber \\
            &\qquad \wedge \neg\exists p\mathbf{:}o \left[ a = \mi{e\_pick}(p) \right]
            \nonumber
\end{align}

Initially, the robot is at $m_1$ and object $o_1$ is at $m_2$. Only $m_1$ is $\mi{Spacious}$, which we will use in
\autoref{sec:platform-models} as a requirement for arm calibration:
\begin{gather}
  \label{eqn:init}
  \Sigma_0 = \{
    \forall x\mathbf{:}o. \, \mi{RAt}(x) \equiv \left(x = m_1\right),
    \\ \nonumber
    \forall x\mathbf{:}o \, \forall y\mathbf{:}o. \, \mi{At}(x,y) \equiv \left(x = o_1 \wedge y = m_2\right),
    \\ \nonumber
    \forall x\mathbf{:}o. \, \mi{Spacious}(x) \equiv \left(x = m_1\right),
    \\ \nonumber
    \tau_o(x) \equiv \left(x = m_1 \vee x = m_2 \vee x = o_1\right),
    \\ \nonumber
    \tau_a(a) \equiv \left( a = \mi{s\_goto}(m_1, m_2) \vee \ldots \vee a = \mi{e\_pick}(o_1)\right)
  \}
\end{gather}

\begin{lstlisting}[language=golog,mathescape=true,columns=flexible,captionpos=below,frame=tb,float,caption={An abstract program to fetch an object.},label={lst:abstract-cleanup}]
$\pi l_r.\:\mi{RAt}(l_r)?; \pi o.\, \pi l_o.\, \mi{At}(o,l_o)?$;
  $\dogolog{goto}{l_r, l_o}; \dogolog{pick}{o};$
\end{lstlisting}

\noindent \autoref{lst:abstract-cleanup} shows a simple program that picks up one object.

\section{\mtl{} Synthesis}\label{sec:mtl-synthesis}
Timed automata~(\acs{TA}\acused{TA}) \cite{alurTheoryTimedAutomata1994,alurTimedAutomata1999} are a widely used model
for representing real-timed systems.  Their properties are often described with
\mtl{}~\cite{koymansSpecifyingRealtimeProperties1990}, a temporal logic that extends \ac{LTL} with metric time.  We
first summarize timed automata and \mtl{}, and then define the problem of controlling a \ac{TA} against an \ac{MTL}
specification, following \cite{bouyerControllerSynthesisMTL2006,ouaknineRecentResultsMetric2008}.

\paragraph{\ac{MTL}}

\ac{MTL} extends \ac{LTL} with timing constraints on the \emph{Until} modality.  One commonly used semantics for
\ac{MTL} is a \emph{pointwise semantics}, in which formulas are interpreted over timed words.

\begin{definition}[Timed Words]
  A timed word $\rho$ over a finite set of atomic propositions $P$ is a finite or infinite sequence
  $\left(\sigma_0,\tau_0\right)\left(\sigma_1,\tau_1\right)\ldots$
  where $\sigma_i \subseteq P$ and $\tau_i \in \mathbb{Q}_+$ such that the sequence $(\tau_i)$ is monotonically
  non-decreasing and non-Zeno.
  The set of timed words over $P$ is denoted as $TP^*$.
\end{definition}

For a timed word $\rho = \left(\sigma_0, t_0\right)\left(\sigma_1, t_1\right)\ldots$ and every $k \in \mathbb{N}$ with
$k \leq |\rho|$, we also write $\rho_k$ for the prefix $\left(\sigma_0, t_0\right)\ldots\left(\sigma_k, t_k\right)$.

\begin{definition}[Formulas of \mtl{}]
  Given a set $P$ of atomic propositions, the formulas of \mtl{} are built as follows:
  \[
    \phi ::= p \smid \neg \phi \smid \phi \wedge \phi \smid \phi \until{I} \phi
  \]
\end{definition}

We use the same abbreviations as for \tesg{}, i.e.,
$\fut{I} \phi \eqdef (\top \until{I} \phi)$ (\emph{future}) and
$\glob{I}\phi \eqdef \neg \fut{I} \neg \phi$
(\emph{globally}). As in \tesg{}, we may omit the interval $I$ if $I = [0, \infty)$.
For a given set of atomic propositions $P$, we denote the language of MTL formulas over $P$ as $\mathcal{L}_{\text{MTL}}(P)$.

\begin{definition}[Pointwise semantics of \mtl{}]\label{def:mtl-semantics}
  Given a timed word $\rho = \left(\sigma_0, \tau_0\right) \left(\sigma_1, \tau_1\right) \ldots$ over
  alphabet $P$ and an \mtl{} formula $\phi$, $\rho, i \models \phi$ is defined
  as follows:
  \begin{enumerate}
    \item $\rho, i \models p$ iff $p \in \sigma_i$
    \item $\rho, i \models \neg \phi$ iff $\rho, i \not\models \phi$
    \item $\rho, i \models \phi_1 \wedge \phi_2$ iff $\rho_i \models \phi_1$ and $\rho_i \models \phi_2$
    \item $\rho, i \models \phi_1 \until{I} \phi_2$ iff there exists $j$ such that
      \begin{enumerate}
        \item $i < j < |\rho|$,
        \item $\rho, j \models \phi_2$,
        \item $\tau_j - \tau_i \in I$,
        \item and $\rho, k \models \phi_1$ for all $k$ with $i < k < j$.
      \end{enumerate}
  \end{enumerate}
\end{definition}
For an \ac{MTL} formula $\phi$, we also write $\rho \models \phi$ for $\rho, 0 \models \phi$
and we define the language of $\phi$ as $L(\phi) = \{ \rho \mid \rho \models \phi \}$.

\paragraph{Alternative definition of \ac{MTL}}
A commonly used alternative definition of \ac{MTL}, especially in the context of timed automata, requires the symbols in
timed words to be from $P$ instead of $2^P$, i.e.,
for a timed word $\rho = \left(\sigma_0, \tau_0\right)\left(\sigma_1, \tau_1\right)\ldots$ over P, we require $\sigma_i
\in P$ (instead of $\sigma_i \subseteq P$). Also, truth of an atomic formula $p$ is defined as:
\begin{enumerate}
 \item[1'.] $\rho, i \models p$ iff $\sigma_i = p$.
\end{enumerate}
Intuitively, a timed automaton describes a transition system with actions leading from one state to the other, where
formulas describe the occurrence of actions, e.g., $\glob{}[ a_1 \supset \fut{} a_2 ]$ says that whenever action $a_1$
occurs, $a_2$ will occur afterwards eventually. Here, the set of atomic propositions $P$ is the set of possible actions.
At most one action may occur at any point in time. Thus,
each $\sigma_i \in P$ defines the action that
occurs at time $\tau_i$.

In our context, formulas describe states of the world, e.g., $\mi{RAt}(m_1) \wedge \mi{Holding}(o_1)$ says that the
robot is at $m_1$ and currently holding $o_1$. Here, the set of atomic propositions is the set of primitive formulas
describing possible world states and multiple predicates may be true at the same time. Thus, each $\sigma_i \subseteq P$
describes the primitive formulas that are true at time $\tau_i$.

Let \ac{MTL}\textsubscript{$\in$} and  denote \ac{MTL} with the alternative semantics and $\models_{\in}$ satisfiability
in \ac{MTL}\textsubscript{$\in$}.
We can define mappings
between \ac{MTL} and
\ac{MTL}\textsubscript{$\in$}. The mapping $\cdot^*: \mathcal{L}_\text{MTL}(P) \rightarrow
\mathcal{L}_{\text{MTL}_\in}(2^P)$ maps a formula of \ac{MTL} into \ac{MTL}\textsubscript{$\in$}, where:
\begin{align*}
  p^* &= \bigvee_{\{ Q \subseteq P \mid p \in Q \}} Q
  \\
  \left(\neg \phi\right)^* &= \neg \phi^*
  \\
  \left(\phi \wedge \psi\right)^* &= \phi^* \wedge \psi^*
  \\
  \left(\phi \until{I} \psi\right)^* &= \phi^* \until{I} \psi^*
\end{align*}
Note that if $\phi$ is a formula over $P$, then $\phi^*$ is a formula over $2^P$, i.e., the atomic propositions in
$\phi^*$ are sub-sets of $P$. As an example, for $P = \{ a, b, c\}$:
$\left(a \wedge b\right)^* =
  \left(\left\{a\right\} \vee \left\{a, b\right\} \vee \left\{a, b, c\right\} \vee \left\{a, c\right\}\right) \wedge
  \left(\left\{b\right\} \vee \left\{a, b\right\} \vee \left\{a,b, c\right\} \vee \left\{b, c\right\}\right)$.

The mapping
$\cdot^+: \mathcal{L}_{\text{MTL}_\in}(P) \rightarrow \mathcal{L}_\text{MTL}(P)$ maps a formula of
\ac{MTL}\textsubscript{$\in$} into \ac{MTL} by enforcing that each $\sigma_i$ contains exactly one symbol from $P$:
\[\phi^+ = \phi \wedge \glob{} \bigvee_{p \in P} \left(p \wedge \bigwedge_{q \in P \setminus \{ p \}} \neg q\right)\]

\begin{theorem}\label{thm:alternative-mtl-semantics}
  For every $\phi \in \mathcal{L}_\text{MTL}(P)$ and $\psi \in \mathcal{L}_{\text{MTL}_\in}(P)$:
\begin{align*}
  {\models} \phi &\Leftrightarrow {\models_{\in}} \phi^*
  \\
  {\models} \psi^+ &\Leftrightarrow {\models_{\in}} \psi
\end{align*}
\end{theorem}
%

In the following, we will use the semantics from \autoref{def:mtl-semantics}. However, related work on MTL synthesis
uses the other formalism. In particular, \autoref{thm:mtl-synthesis-finite} uses the the alternative MTL semantics from
above. With \autoref{thm:alternative-mtl-semantics}, we can apply those results while using the semantics from
\autoref{def:mtl-semantics}.

\paragraph{\ac{MTL} and \tesg{}}
Timed words in \ac{MTL} are similar to traces in \tesg{}. In fact, \tesg{} subsumes \ac{MTL}:
\begin{theorem}[\citet{hofmannLogicSpecifyingMetric2018}]
  Let $\phi$ be a sentence of \mtl. Then $\models_{\tesg{}} \phi$ iff $\models_{\text{MTL}} \phi$.
\end{theorem}

\paragraph{Symbolic transition systems and timed automata}
Intuitively, a timed automaton is a finite automaton extended with time. More specifically, a timed automaton has a
finite set of clocks; time may pass in the vertices of the graph, which are also called \emph{locations}. Transitions,
also called switches, are the edges of the graph. They are always instantaneous, may have clock constraints, and may
reset some clocks to zero.
Formally, we first define \acfp{STS}:

\begin{definition}[Symbolic Transition Systems and Timed Automata~\cite{bouyerControllerSynthesisMTL2006}]
Let $X$ be a finite set of variables (called \emph{clocks}).
The set $\mathcal{G}(X)$ of \emph{clock constraints} $g$ over $X$ is defined by the grammar $g ::= g \wedge g \mid x \bowtie c$, where $\mathord{\bowtie} \in \{ <, \leq, =, \geq, > \}$, $x \in X$, and $c \in \mathbb{Q}_{\geq 0}$.
A \emph{valuation} over $X$ is a mapping $\nu : X \rightarrow \mathbb{R}_{\geq 0}$. The set of valuations satisfying a
constraint $g$ is denoted as $\llbracket g \rrbracket$.
A granularity is defined by a triple $\mu = \left(X, m, K\right)$, where $X$ is a finite set of clocks, $m \in \mathbb{N}_{> 0}$, and $K \in \mathbb{N}$.
A constraint $g$ is $\mu$-granular if it only uses clocks from $X$ and each constant in $g$ is $\frac{\alpha}{m}$ with $\alpha \leq K$ and $\alpha \in \mathbb{N}$.

For alphabet $P$ and clocks $X$, a \emph{symbolic alphabet} $\Gamma$ is a finite subset of $2^P \times \mathcal{G}(X) \times 2^X$, where a symbolic action $\left(p,g,Y\right) \in \Gamma$ is interpreted as action $p$ can happen if the constraint $g$ is satisfied, with the clocks in $Y$ being reset after the action.
A symbolic word $\gamma = \left(a_1, g_1, Y_1\right)\left(a_2, g_, Y_2\right)\ldots$ over $\Gamma$ gives rise to a set of
timed words $\mi{tw}(\gamma)$ over $P$.

A \emph{\acf{STS}} over a symbolic alphabet $\Gamma$ based on $\left(P, X\right)$ is a tuple $\mathcal{T} = \left(S, s_0, \rightarrow, F\right)$, where $S$ is a possibly infinite set of states, $s_0 \in S$ is the initial state, $\mathnormal{\rightarrow} \subseteq S \times \Gamma \times S$ is the transition relation, and $F \subseteq S$ is a set of accepting state.
The \emph{timed language} accepted by an \ac{STS} $\mathcal{T}$ is denoted as $L(\mathcal{T})$.

A \ac{STS} is called \emph{deterministic} if there are no distinct transitions
$q \overset{a, g_1, Y_1}{\longrightarrow} q_1$ and $q \overset{a, g_2, Y_2}{\longrightarrow} q_2$ with
$ \llbracket g_1 \rrbracket \cap \llbracket g_2 \rrbracket \neq \emptyset$.

A \emph{\acf{TA}} is an \ac{STS} with finitely many states.
\end{definition}

We also want to compose \acp{STS}:
\begin{definition}[\ac{STS} Compositions]
  For two \ac{STS} $\mathcal{T}_1 = \la Q_1, q_0^1, \rightarrow_1, F_1 \ra$ over $\Gamma_1$ based on
  $\left(P_1, X_1\right)$ and $\mathcal{T}_2 = \la Q_2, q_0^2,
  \rightarrow_2, F_2 \ra$ over $\Gamma_2$ based on $\left(P_2, X_2\right)$,
the \emph{parallel composition} $\mathcal{T}_1 \parallel \mathcal{T}_2$ of $\mathcal{T}_1$ and
$\mathcal{T}_2$ is the \ac{STS} $\la Q, q_0, \rightarrow, F \ra$ where $Q = Q_1 \times Q_2$, $q_0 = \left(q_0^1,
q_0^2\right)$, $F = F_1 \times F_2$ and $\left(p_1, p_2\right) \overset{a,g,Y}{\longrightarrow} \left(q_1, q_2\right)$
iff $p_1 \overset{a,g_1,Y_1}{\longrightarrow} q_1$ and $p_2 \overset{a,g_2,Y_2}{\longrightarrow} q_2$ with $g = g_1
\wedge g_2$ and $Y = Y_1 \cup Y_2$.

If $P_1 \cap P_2 = \emptyset$, then the \emph{product \ac{STS}} $\mathcal{T}_1 \times
\mathcal{T}_2$ is the \ac{STS} $\la Q, q_0, \rightarrow, F \ra$ where $Q = Q_1 \times Q_2$, $q_0 = \left(q_0^1,
q_0^2\right)$, $F = F_1 \times F_2$ and $\left(p_1, p_2\right) \overset{a,g,Y}{\longrightarrow} \left(q_1, q_2\right)$
iff $p_1 \overset{a_1,g_1,Y_1}{\longrightarrow} q_1$, $p_2 \overset{a_2,g_2,Y_2}{\longrightarrow} q_2$, and $a = a_1
\cup a_2$, $g = g_1 \wedge g_2$, and $Y = Y_1 \cup Y_2$.
\end{definition}

In the parallel composition $\mathcal{T}_1 \parallel \mathcal{T}_2$, both $\mathcal{T}_1$ and
$\mathcal{T}_2$ take a transition for the same input simultaneously. The product $\mathcal{T}_1 \times
\mathcal{T}_2$ takes a transition on a symbol $a$ if $a$ is the union $a = a_1 \cup a_2$ of two input symbols
$a_1$ and $a_2$, such that $\mathcal{T}_1$ ($\mathcal{T}_2$) can take a transition on $a_1$ ($a_2$).

\paragraph{\ac{MTL} Control Problem}
Finally, we define the \ac{MTL} control problem. Intuitively, the goal is to synthesize a controller $\mathcal{C}$ that
\emph{controls} a \emph{plant} $\mathcal{P}$ against a specification of desired behaviors $\Phi$ such that all resulting
traces satisfy the specification $\Phi$ without blocking the plant $\mathcal{P}$. In this context,
\emph{control} means that $\mathcal{C}$ has control over some actions, while the environment controls the remaining
actions. Formally:
\begin{definition}[\ac{MTL} Control Problem \cite{bouyerControllerSynthesisMTL2006}]
  Let $P = P_C \cup P_E$ be an alphabet partitioned into a set of \emph{controllable} actions $P_C$ and a set of \emph{environment} actions $P_E$.
  A \emph{plant} $\mathcal{P}$ over $P$ is a deterministic \ac{TA}.
  Let the clocks used in $\mathcal{P}$ be $X_\mathcal{P}$ and $\mu = \left(X_\mathcal{P} \cup X_\mathcal{C}, m, K\right)$ be a granularity finer than that of the plant.
  Then, a $\mu$-controller for $\mathcal{P}$ is a deterministic \ac{STS} $\mathcal{C}$ over a symbolic alphabet based on $\left(P, X_\mathcal{P} \cup X_\mathcal{C}\right)$ having granularity $\mu$ and satisfying:
  \begin{enumerate}
    \item $\mathcal{C}$ does not reset the clocks of the plant: $q_\mathcal{C} \overset{a, g, Y}{\longrightarrow}
      q'_\mathcal{C}$ implies $Y \subset X_\mathcal{C}$,
    \item $\mathcal{C}$ does not restrict environment actions:
      if $\sigma \in L(\mathcal{P} \parallel \mathcal{C})$ and $\sigma \left(e, t\right) \in
        L(\mathcal{P})$ with $e \in P_E$,
        then $\sigma \cdot \left(e, t\right) \in L(\mathcal{P} \parallel \mathcal{C})$
    \item $\mathcal{C}$ is non-blocking:
      if $\sigma \in L(\mathcal{P} \parallel \mathcal{C})$ and $\sigma \left(a, t\right) \in
      L(\mathcal{P})$ and $\sigma \cdot \left(a, t\right) \in L(\mathcal{P})$,
      then $\sigma \cdot \left(b, t'\right) \in \mathcal{L}^*(\mathcal{P} \parallel \mathcal{C})$ for some
      $b \in P$ and $t' \in \mathbb{Q}$
    \item all states of $\mathcal{C}$ are accepting.
  \end{enumerate}
  For a timed language $\mathcal{L} \subseteq TP^*$, we say that a $\mu$-controller $\mathcal{C}$ \emph{controls}
  $\mathcal{P}$ against the specification of desired behaviors $\Phi$ iff $L(\mathcal{P} \parallel \mathcal{C})
  \subseteq L(\Phi)$.
  The control problem with fixed resources against desired behaviors is to decide, given a plant $\mathcal{P}$, a
  set of formulas $\Phi$, and a granularity $\mu$ finer than that of $\mathcal{P}$, whether there exists a
  $\mu$-controller $\mathcal{C}$ which controls $\mathcal{P}$ against the specification of desired behaviors
  $\Phi$.
\end{definition}

\citeauthor{bouyerControllerSynthesisMTL2006} showed that the synthesis problem is decidable, with some restrictions:
\begin{theorem}[\citet{bouyerControllerSynthesisMTL2006}]\label{thm:mtl-synthesis-finite}
  The control problem for fixed resources against \ac{MTL} specifications over finite words representing desired behaviors is decidable.
  Moreover, if there exists a controller, then one can effectively construct a finite-state one.
\end{theorem}

%
We will use this result by constructing a \ac{TA} $\mi{PTA}(\Sigma, \delta)$ from a determinate \ac{fd-BAT} $\Sigma$ and
program $\delta$, modelling the platform as another \ac{TA} $\mathcal{R}$, and synthesizing a controller $\mathcal{C}$ that controls
the \ac{TA} $\mathcal{T} = \mi{PTA}(\Sigma, \delta) \times \mathcal{R}$ against the platform constraints $\Phi$.

\section{Constructing a \ac{TA} from a Program}\label{sec:pta}

We describe how to construct a \ac{TA} from a program $\delta$ over a determinate \ac{fd-BAT} $\Sigma$.
We do this by using
$P = \mathcal{P}_\Sigma \cup A_\Sigma$ as alphabet for the \ac{TA} $\mi{PTA}(\Sigma, \delta)$, i.e., the alphabet $P$
consists of all primitive formulas and action standard names from $\Sigma$.
In each transition, we encode the occurring action and the resulting situation, such that $p \overset{\sigma,
\emptyset, \emptyset}{\rightarrow} q$ for $\sigma = \{ f_1, \ldots, f_k, a \}$ if after doing action $a \in
\mathcal{A}_\Sigma$ in the corresponding situation, exactly the primitive formulas $\{ f_1, \ldots, f_k \} \subseteq
\mathcal{P}_\Sigma$ are true. By doing so, we obtain a correspondence of traces of the program $\delta$ with traces in
the \ac{TA}.

We assume that $\Sigma$ is a determinate \acl{fd-BAT} and $\delta$ is a program over
$\Sigma$.
We need to restrict $\Sigma$ to be a determinate \ac{BAT} as in the resulting timed automaton, each
transition encodes which primitive formulas are true in the respective situation. In particular, the transition $q_0 \rightarrow
S_0$ will encode the primitive formulas that are true in the initial situation. As we cannot encode disjunctions in such a
transition, we need $\Sigma_0$ to determine the truth for each primitive formula $f_i$. Also, as each transition can only contain
finitely many symbols, $\Sigma$ needs to be restricted to a finite domain. Furthermore, we assume that $\delta$ is
terminating, i.e., it only induces finite traces, which is necessary to guarantee that the resulting transition system
indeed has a finite number of states. We will further discuss those restrictions in \autoref{sec:conclusion}.

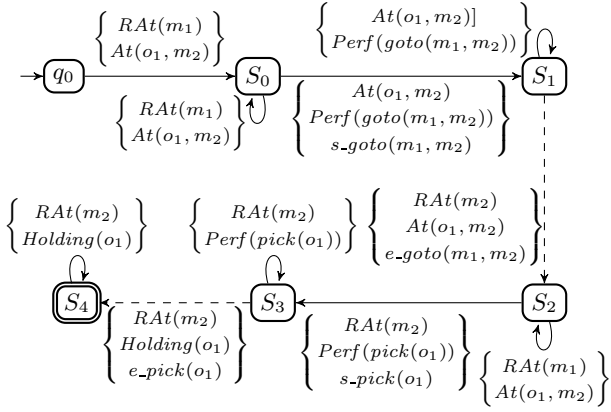
\begin{figure}[ht]
  \centering
  \begin{tikzpicture}
    \node[stn] (s) {$q_0$};
    \node[draw=none,left=3mm of s] (init) {};
    \draw[->] (init) -- (s);
    \node[stn,right=2.0cm of s] (s0) {$S_0$} edge [loop below] node[edgelabel,left]{
      $\begin{Bmatrix}
        \mi{RAt}(m_1) \\
        \mi{At}(o_1,m_2)
      \end{Bmatrix}$
      } ();
    \draw[->] (s) -- (s0) node [edgelabel, above, midway]{
      $\begin{Bmatrix}
        \mi{RAt}(m_1) \\
        \mi{At}(o_1,m_2)
      \end{Bmatrix}$
      };
    \node[stn,right=3.2cm of s0] (goto1) {$S_1$} edge [loop above] node[edgelabel,left]{
      $\begin{Bmatrix}
        \mi{At}(o_1,m_2)] \\
        \mi{Perf}(\mi{goto}(m_1, m_2))
      \end{Bmatrix}$
      } ();
    \draw[->] (s0) -- (goto1) node [edgelabel, below=-1mm, midway]{
      $\begin{Bmatrix}
        \mi{At}(o_1,m_2) \\
        \mi{Perf}(\mi{goto}(m_1, m_2)) \\
        \mi{s\_goto}(m_1, m_2)
      \end{Bmatrix}$
      };
    \node[stn,below=2.5cm of goto1] (agoto1) {$S_2$} edge [loop below] node[edgelabel,below=-1mm]{
      $\begin{Bmatrix}
        \mi{RAt}(m_1) \\
        \mi{At}(o_1,m_2)
      \end{Bmatrix}$
      } ();
    \draw[->,dashed] (goto1) -- (agoto1) node [edgelabel,left=-1.5mm,pos=0.7]{
      $\begin{Bmatrix}
        \mi{RAt}(m_2) \\
        \mi{At}(o_1,m_2) \\
        \mi{e\_goto}(m_1, m_2)
      \end{Bmatrix}$
      };
    \node[stn,left=3.0cm of agoto1] (pick1) {$S_3$} edge [loop above] node[edgelabel,above=-1mm]{
      $\begin{Bmatrix}
        \mi{RAt}(m_2) \\
        \mi{Perf}(\mi{pick}(o_1))
      \end{Bmatrix}$
      } ();
    \draw[->] (agoto1) -- (pick1) node [edgelabel,below,pos=0.6]{
      $\begin{Bmatrix}
        \mi{RAt}(m_2) \\
        \mi{Perf}(\mi{pick}(o_1)) \\
        \mi{s\_pick}(o_1)
      \end{Bmatrix}$
      };
    \node[double,stn,left=2.0cm of pick1] (f) {$S_4$} edge [loop above] node[edgelabel,above=-1mm]{
      $\begin{Bmatrix}
        \mi{RAt}(m_2) \\
        \mi{Holding}(o_1)
      \end{Bmatrix}$
      }();
    \draw[->,dashed] (pick1) -- (f) node [edgelabel,below=-1mm,midway]{
      $\begin{Bmatrix}
        \mi{RAt}(m_2) \\
        \mi{Holding}(o_1) \\
        \mi{e\_pick}(o_1)
      \end{Bmatrix}$
      };
  \end{tikzpicture}
  \caption{The \ac{TA} for the program from \autoref{lst:abstract-cleanup} and the initial situation from
  \autoref{eqn:init}. The dashed edges are controlled by the environment.}
  \label{fig:plan}
\end{figure}

\begin{definition}[Program Timed Automata]\label{def:pta}
  Given a program $\delta$ over a determinate \ac{fd-BAT} $\Sigma$.
  We define the timed automaton $\mi{PTA}(\Sigma, \delta) = \left(S, q_0, \rightarrow, F\right)$ as follows:
  \begin{enumerate}
    \item $q_0 \overset{P, \emptyset, \emptyset}{\longrightarrow} \left(\la\ra, \delta\right)$
      with $P = \{ f_i \in \mathcal{P}_\Sigma \mid w_\Sigma[ f_i, \la\ra] = 1 \}$
      \label{def:pta:init}
    \item $\left(z,\delta\right) \overset{P \cup \{ a \}, \emptyset, \emptyset }{\longrightarrow}
      \left(z \cdot a, \delta' \right)$ iff
      $\left(z^0,\delta\right) \overset{w_\Sigma}{\rightarrow} \left(\left(z \cdot a\right)^0, \delta'\right)$
      and $P = \{ f_i \in \mathcal{P}_\Sigma \mid w_\Sigma[f_i, \left(z \cdot a\right)^0] = 1 \} $
      \label{def:pta:action}
    \item $\left(z, \delta\right) \overset{P, \emptyset, \emptyset}{\longrightarrow} \left(z, \delta\right)$ with
      $P = \{ f_i \in \mathcal{P}_\Sigma \mid w_\Sigma[f_i, z] = 1 \} $
      \label{def:pta:self-loop}
    \item $\left(z, \delta\right) \in F$ iff $\la z^0, \delta \ra \in \mathcal{F}^{w_\Sigma}$
      \label{def:pta:final}
  \end{enumerate}
\end{definition}


A word $\rho$ of the \ac{TA} $\mi{PTA}(\Sigma, \delta)$ corresponds to a trace $\tau \in \|\delta\|_{w_\Sigma}$. We can
map $\rho$ to $\tau$:
\begin{definition}[Induced action trace]
  Given a word $\rho \in \mi{PTA}(\Sigma, \delta)$, we define the \emph{(action) trace $\mu(\rho)$ induced by $\rho$} inductively:
  \begin{itemize}
    \item If $\rho = \la\ra$, then $\mu(\rho) = \la\ra$
    \item If $\rho = \left(\{\ldots, a_i \}, t_i \right) \cdot \rho'$ for some action standard name $a_i \in A_\Sigma$,
      then $\mu(\rho) = \left(a_i, t_i\right) \cdot \mu(\rho')$
    \item Otherwise, if $\rho = \left(\sigma_i, t_i \right) \cdot \rho'$ and $\sigma_i \cap A_\Sigma = \emptyset$
      (i.e., $\sigma_i$ contains no action from $\Sigma$), then $\mu(\rho) = \mu(\rho')$
  \end{itemize}
\end{definition}



The trace $\mu(\rho)$ induced by an MTL word $\rho \in \mi{PTA}(\Sigma, \delta)$ is indeed a trace of the program:
\begin{lemma}\label{lma:trace-existence}
  Given a program $\delta$ over a determinate \ac{fd-BAT} $\Sigma$.
  Then:
  \begin{enumerate}
    \item For every $\rho \in L(\mi{PTA}(\Sigma, \delta))$: $\mu(\rho) \in \|\delta\|_{w_\Sigma}$.
    \item For every $\tau \in \|\delta\|_{w_\Sigma}$, there is a $\rho \in L(\mi{PTA}(\Sigma, \delta))$ such that
  $\mu(\rho) = \tau$.
  \end{enumerate}
\end{lemma}

\begin{proof}
    Follows directly from the construction of $\mi{PTA}(\Sigma, \delta)$ and \autoref{lma:0-equivalence}.
\end{proof}


Furthermore, we can show that the \ac{MTL} word $\rho$ and the trace $\mu(\rho)$ entail the
same fluent state formulas at every point in time:
\begin{theorem}\label{thm:plan-ta-equivalence}
  Given a program $\delta$ over a determinate \ac{fd-BAT} $\Sigma$.
  Then:
  \begin{enumerate}
    \item For every $\rho \in L(\mi{PTA}(\Sigma, \delta))$ and every $k \leq |\rho|$,
      there is a $\tau = z \cdot \tau' \in \|\delta\|_{w_\Sigma}$ such that $\mu(\rho_k) = z$ and
      \[
        w_\Sigma, z \models \alpha \Leftrightarrow \rho_k \models \alpha
      \]
    \item For every $\tau \in \|\delta\|_{w_\Sigma}$ and every $z$ with $\tau = z \cdot \tau'$, there is a
      $\rho \in L(\mi{PTA}(\Sigma, \delta))$ such that for some $i \leq |\rho|$, $\mu(\rho_k) = z$ and
      \[
        w_\Sigma, z \models \alpha \Leftrightarrow \rho_k \models \alpha
      \]
  \end{enumerate}
\end{theorem}

\begin{proof}
  ~
  \begin{enumerate}
    \item Let $\rho \in L(\mi{PTA}(\Sigma, \delta))$. By \autoref{lma:trace-existence}, we know that $\tau\left(\rho\right) \in \|\delta\|_{w_\Sigma}$.
      It remains to be shown that for every $k \leq |\rho|$, there is a $z, \tau'$ such that
      $\tau = z \cdot \tau'$ and $\mu(\rho_k) = z$.
      We show the existence of $z, \tau'$ by induction over $k$:
      \begin{enumerate}
        \item Let $k = 0$. Thus $\rho_k = \left(\sigma_0, t_0\right)$. By definition of $\mi{PTA}(\Sigma, \delta)$, we know
          that $\sigma_0 = \Sigma_0$. For $z = \la\ra$, it follows that $\mu(\rho_k) = z$ and
          $w_\Sigma, z \models \alpha \Leftrightarrow w_\Sigma \models \alpha \Leftrightarrow \Sigma_0 \models \alpha
          \Leftrightarrow \rho' \models \alpha$.
        \item Let $k = l + 1$. By induction, there is a $z'$ such that $\tau = z' \cdot \tau'$,
          $z' = \mu(\rho_l)$, and
          $w_\Sigma, z' \models \alpha \Leftrightarrow \rho_l \models \alpha$.
          Now, we have two cases:
          \begin{enumerate}
            \item There is some action symbol $a \in \sigma_k$. Then, by definition of $\mi{PTA}(\Sigma, \delta)$,
              for $z = z' \cdot \left(a, t_k\right)$,
              $w_\Sigma, z \models \alpha \Leftrightarrow \rho_k \models \alpha$.
            \item There is no action symbol in $\sigma_k$. Then, by definition of $\mi{PTA}(\Sigma, \delta)$,
              $\sigma_k = \{ f_i \mid w_\Sigma[f_i, z'] = 1 \}$ and thus, for $z = z'$, it follows that
              $w_\Sigma, z \models \alpha \Leftrightarrow \rho_k \models \alpha$.
          \end{enumerate}
      \end{enumerate}
    \item Let $\tau \in \|\delta\|_{w_\Sigma}$. By \autoref{lma:trace-existence}, we know that there is a $\rho \in
      L(\mi{PTA}(\Sigma, \delta))$.
      It remains to be shown that for every $z$ with $\tau = z \cdot \tau'$, $\mu(\rho_k) = z$ and
      $w_\Sigma, z \models \alpha \Leftrightarrow \rho_k \models \alpha$.
      By induction over the length $i$ of $z$:
      \begin{enumerate}
        \item Let $i = 0$, i.e., $z = \la\ra$, and thus $w_\Sigma, z \models \alpha$ iff $\Sigma_0 \models \alpha$.
          By definition of $\mi{PTA}(\Sigma, \delta)$, $\rho_0 = \left(\Sigma_0, t_0\right)$ for some $t_0$. Thus,
          $\mu(\rho_0) = \la\ra$ and $\rho_0 \models \alpha$ iff $\Sigma_0 \models \alpha$.
        \item Let $i = j + 1$, i.e., $z = z' \cdot \left(a_i, t_i\right)$.
          By induction, $z' = \mu(\rho_l)$ for some $l$ and
            $w_\Sigma, z' \models \alpha \Leftrightarrow \rho_l \models \alpha$.
            By definition of $\mi{PTA}(\Sigma, \delta)$:
            \[
              \rho = \overbrace{\underbrace{\left(\sigma_0, t_0\right)\ldots\left(\sigma_l, t_l\right)}_{\rho_l}\left(\sigma_{l+1},
            t_{l+1}\right) \ldots (\underbrace{\{ \ldots, a_i \}}_{\sigma_k}, t_k \})}^{\rho_k}
            \]
            where none of $\sigma_{l+1},\ldots,\sigma_{k-1}$ contains any action symbol. Then, by definition of
            $\mi{PTA}(\Sigma, \delta)$, $\sigma_k = \{ f \mid w_\Sigma[f, z] = 1\}$, and thus
            $w_\Sigma, z \models \alpha \Leftrightarrow \rho_k \models \alpha$.
      \end{enumerate}
  \end{enumerate}
\end{proof}

\section{Platform Models}\label{sec:platform-models}

\begin{figure}
  \centering
  \begin{tikzpicture}[every node/.append style={stn},node distance=12mm]
    \node[double] (init) {$\mi{Init}$} edge [loop above] node[edgelabel] {$\{\mi{Ready}\}$}  ();
    \node[right=1.7cm of init] (calibrating) {$\mi{Calibrating}$} edge [loop above] node[edgelabel]
      {$\{\mi{Calibrating}\}$} ();
    \node[double,right=1.7cm of calibrating] (calibrated) {$\mi{Calibrated}$} edge [loop below] node[edgelabel]
      {$\{\mi{Calibrated}\}$} ();
    \node[draw=none,left=3mm of init] (init0) {};
    \draw[->](init0) -- (init);
    \draw[->] (init) -- (calibrating) node [edgelabel,below,midway] {
      \shortstack{
        $\begin{Bmatrix} \mi{Calibrating} \\ \mi{s\_calibrate} \end{Bmatrix}$
        \\ $t_p:=0$}
    };
    \draw[->,dashed] (calibrating) -- (calibrated) node [edgelabel,above,midway]{
      \shortstack{
        $\begin{Bmatrix} \mi{Calibrated} \\ \mi{e\_calibrate} \end{Bmatrix}$
        \\ $t_p = 5$} 
    };
  \draw[->] (calibrated) edge[bend left] node [edgelabel,below,midway] {
      \shortstack{
        $\begin{Bmatrix} \mi{Calibrating} \\ \mi{s\_calibrate} \end{Bmatrix}$
        \\ $t_p:=0$}
    } (calibrating) ;
  \end{tikzpicture}
  \caption{The platform model of a robot arm.}
  \label{fig:parking-arm}
\end{figure}
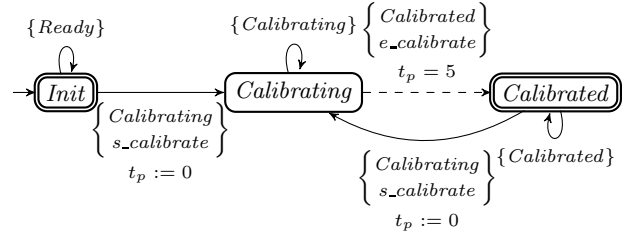

We model the robot platform with timed automata, an example is shown in \autoref{fig:parking-arm}.  Similar to PTAs, we
expect a platform model to use an alphabet with symbols of the form $\{ f_1, \ldots f_k, a \}$, where $a \in
\mathcal{N}_A \setminus A_\Sigma$ is a platform action and $f_i \in \mathcal{P}_F \setminus \mathcal{P}_\Sigma$ are
exactly those primitive formulas that are true after executing the action. We expect $f_i$ and $a$ to be from a
different alphabet than the \ac{BAT}, i.e., the platform does not have any effects on the abstract program and vice
versa. Further, to guarantee that the platform model does not block the PTA, we expect it to contain self loops, similar
to the self loops of a PTA, and as shown in \autoref{fig:parking-arm}.

\paragraph{Platform Constraints}
Given a determinate \ac{fd-BAT} $\Sigma$ and a platform model $\mathcal{R}$, we can formulate constraints over
$\Sigma$ and $\mathcal{R}$:
\begin{align}
  \label{eqn:pick-requires-calibrated-arm}
  &\glob{} \neg \mi{Calibrated} \supset \neg \fut{\leq 10} \exists p\mathbf{:}o.\, \mi{Perf}(\mi{pick}(p))
  \\
  \label{eqn:calibrate-with-enough-space}
  &\glob{} \mi{Calibrating} \supset \exists l\mathbf{:}o.\, \mi{RAt}(l) \wedge \mi{Spacious}(l)
\end{align}
The first constraint states that if the robot's arm is not calibrated, it must not perform a pick action in the
next 10 seconds, i.e., it must calibrate the arm before doing \texttt{pick}.
The second constraint says that if the robot is calibrating its arm, it must be at a location that provides enough space
for doing so, i.e., a $\mi{Spacious}$ location.

\section{Synthesizing a Controller}\label{sec:synthesis}

\begin{figure}[htb]
  \centering
  \begin{tikzpicture}[every node/.append style={stn,double}]
    \node (s0) {};
    \node[draw=none,left=3mm of s0] (init) {};
    \draw[->] (init) -- (s0);

    \node[right=2cm of init] (scalibrate) {};
    \draw[->] (s0) -- (scalibrate) node [edgelabel, above, midway]{
        $\begin{Bmatrix}
        \mi{RAt}(m_1) \\
        \mi{At}(o_1,m_2) \\
        \mi{Calibrating} \\
        \mi{s\_calibrate}
        \end{Bmatrix}$
      };
    \node[right=2cm of scalibrate] (ecalibrate) {};
    \draw[->,dashed] (scalibrate) -- (ecalibrate) node [edgelabel, above, midway]{
        \shortstack {
        $\begin{Bmatrix}
        \mi{RAt}(m_1) \\
        \mi{At}(o_1,m_2) \\
        \mi{Calibrated} \\
        \mi{e\_calibrate}
        \end{Bmatrix}$
        \\ $t_c := 0$
        }
      };

    \node[right=2cm of ecalibrate] (sgoto) {};
    \draw[->] (ecalibrate) -- (sgoto) node [edgelabel, above,midway]{
        $\begin{Bmatrix}
          \mi{At}(o_1,m_2) \\
          \mi{s\_goto}(m_1, m_2) \\
          \mi{Calibrated}
        \end{Bmatrix}$
    };
    \node[below=2.5cm of sgoto] (egoto) {};
    \draw[->,dashed] (sgoto) -- (egoto)  node [edgelabel, left, midway]{
      $\begin{Bmatrix}
        \mi{RAt}(m_2) \\
        \mi{At}(o_1,m_2) \\
        \mi{e\_goto}(m_1, m_2) \\
        \mi{Calibrated}
      \end{Bmatrix}$
      };
    \node[left=5cm of egoto] (spick) {} edge [loop left] node[edgelabel] {*}();
    \draw[->] (egoto) -- (spick) node [edgelabel, above, near end]{
        \shortstack {
        $\begin{Bmatrix}
        \mi{RAt}(m_2) \\
        \mi{At}(o_1,m_2) \\
        \mi{Perf}(\mi{pick}(o_1)) \\
        \mi{s\_pick}(o_1) \\
        \mi{Calibrated}
        \end{Bmatrix}$
        \\ $t_c > 10$
        }
      };
  \end{tikzpicture}
  \caption{A possible controller that controls the program from \autoref{fig:plan} and the platform from
  \autoref{fig:parking-arm} against the constraints from Equations \ref{eqn:pick-requires-calibrated-arm} and
  \ref{eqn:calibrate-with-enough-space}. The dashed edges are controlled by the environment.}
  \label{fig:controller}
\end{figure}
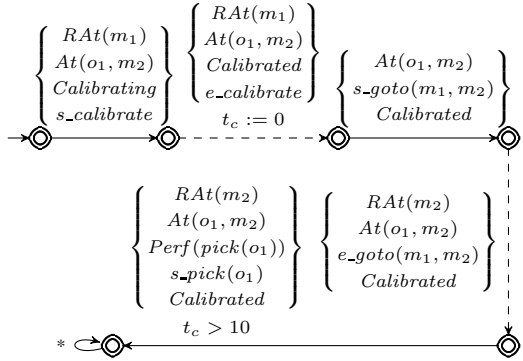

Using the \ac{TA} $\mi{PTA}(\Sigma, \delta)$ that represents the program $\delta$, the \ac{TA} $\mathcal{R}$ for the
platform, and constraints $\Phi$, we can use \ac{MTL} synthesis to synthesize a controller that executes $\delta$ while
satisfying the platform constraints. Specifically, we use
\begin{enumerate}
  \item the plant $\mathcal{P} = \mi{PTA}(\Sigma, \delta) \times \mathcal{R}$,
  \item as controllable actions $P_C$ all symbols that contain start actions of the program or the platform model,
    i.e., $P_C = \{ S \mid S \in P, \mi{s\_a}(\vec{t}) \in S \text{ for some } a(\vec{t}) \}$,
  \item as environment actions $P_E$ all symbols that contain end actions of the program or the platform model,
    i.e., $P_E = \{ E \mid E \in P, \mi{e\_a}(\vec{t}) \in E \}$ for some $a(\vec{t})$,
  \item a fixed granularity $\mu$, e.g.,  based on the robot platform's time resolution
\item the set of MTL formulas $\Phi$ as specification of desired behaviors.
\end{enumerate}

\autoref{fig:controller} shows a possible controller for our example program from \autoref{lst:abstract-cleanup}, the
platform from \autoref{fig:parking-arm}, and the constraints from \autoref{sec:platform-models}.

We can show that
\begin{inparaenum}[(1)]
  \item the resulting controller indeed satisfies the constraints and
  \item each of its traces is equivalent to some trace of the original program, i.e., the resulting controller satisfies
    the same situation formulas as the original program at any point of the execution:
\end{inparaenum}
\begin{theorem}\label{thm:synthesize-controller-properties}
  Let $\Sigma$ be a determinate \ac{fd-BAT}, $\delta$ a program over $\Sigma$ that only induces finite traces,
  $\mathcal{R}$ a platform model with symbols disjunct with the symbols from $\Sigma$, and let the constraints $\Phi$ be
  a set of MTL formulas.
  Let $\mathcal{C}$ be the synthesized \ac{MTL} controller with $\mathcal{L} = L\mleft(\left(\mi{PTA}(\Sigma, \delta)
  \times \mathcal{R}\right) \parallel \mathcal{C}\mright)$.
  Then:
  \begin{enumerate}
    \item $\mathcal{L} \subseteq L(\Phi)$, i.e., all constraints are satisfied.
    \item For every $\rho = \rho' \cdot \rho'' \in \mathcal{L}$, $\mu(\rho) \in \|\delta\|_{w_\Sigma}$, and for every fluent
      state formula restricted to $\Sigma$:
      \[
        \rho' \models \alpha \Leftrightarrow w_\Sigma, \mu(\rho') \models \alpha
      \]
  \end{enumerate}
\end{theorem}

\begin{proof}
  ~
  \begin{enumerate}
    \item Follows directly from \autoref{thm:mtl-synthesis-finite}.
    \item First, note that $\mathcal{L} \subseteq L(\mi{PTA}(\Sigma, \delta) \times \mathcal{R})$. Second, as
      $\mathcal{R}$ does not contain any action standard name from $\Sigma$, for every $\rho \in \mathcal{L}$, there is
      a $\rho' \in \mi{PTA}(\Sigma, \delta)$ such that $\mu(\rho) = \mu(\rho')$. By \autoref{thm:plan-ta-equivalence},
      for every $\rho' \in \mi{PTA}(\Sigma, \delta)$, $\mu(\rho') \in \|\delta\|_{w_\Sigma}$ and $\rho' \models \alpha$
      iff $w_\Sigma, \mu(\rho') \models \alpha$.
      \qedhere
  \end{enumerate}
\end{proof}

Thus, the resulting controller preserves the program's original effects while satisfying all platform constraints.


\section{Conclusion}\label{sec:conclusion}
In this paper, we have described how to synthesize a controller that controls a \golog{} program over a finite domain
against a robot platform with metric temporal constraints. We did so by reducing the problem to the \ac{MTL} synthesis
problem, assuming that the initial state is completely known, the original program does not refer to time and only
induces finite traces. For this
reduction, we generated a \acf{TA} from the initial situation $\Sigma_0$, the program $\delta$ and the platform model
$\mathcal{R}$, where each transition describes all the fluents that are true in the respective situation. We then
synthesized an \ac{MTL} controller that controls the generated \ac{TA} against a set of \ac{MTL} constraints $\Phi$. By
doing so, we obtain a decidable procedure to control an abstract program against a platform model with metric temporal
constraints.

For future work, we plan to implement the proposed synthesis method based on \cite{bouyerControllerSynthesisMTL2006}.

While the restriction to a finite domain is fundamental for the described synthesis method, in future work, we may want
to allow programs that allow infinite traces. This is possible if we restrict the constraints to Safety MTL but requires
modifications to the \ac{TA} representation of the program, as the resulting \ac{TA} must not have infinitely many
states. Furthermore, we may want to allow programs that refer to time, e.g., by defining equivalence classes of traces
that may refer to different points in time but imply the same situation
formulas. Lastly, it would be interesting to go beyond determinate
\acp{BAT} to allow some form of incompleteness, for example, by considering sets of literals under
the open world assumption~\cite{levesqueCompletenessResultReasoning1998}.

\pagebreak
\bibliographystyle{kr}
\bibliography{ConTrAkt}

\end{document}